\documentclass{article}

\usepackage{lipsum}
\usepackage{amsfonts,amsmath,amsthm}
\usepackage{graphicx}
\usepackage{epstopdf}
\usepackage{algorithmic}
\usepackage{tikz}
\usepackage{tikz-cd}
\usepackage{caption}
\usepackage{subcaption}

\usetikzlibrary{positioning}
\tikzset{>=stealth}

\newtheorem{theorem}{Theorem}
\newtheorem{definition}{Definition}
\newtheorem{remark}{Remark}
\newtheorem{problem}{Problem}
\newtheorem{example}{Example}
\newtheorem{assumption}{Assumption}
\newtheorem{lemma}{Lemma}
\newtheorem{proposition}{Proposition}

\title{A Path-Dependent Variational Framework for Incremental Information Gathering  \thanks{This work was funded by DMS-1645643. Funding for M. Ghaffari was in part provided by the Toyota Research Institute (TRI), partly under award number N021515.}}

\author{William Clark\thanks{Department of Mathematics, Cornell University, Ithaca, NY
  ({wac76@cornell.edu}).}
\and Maani Ghaffari\thanks{Department of Naval Architecture \& Marine Engineering, University of Michigan, Ann Arbor, MI 
  ({maanigj@umich.edu.edu}).}
}
\usepackage{amsopn}

\begin{document}

\maketitle

\begin{abstract}
  Information gathered along a path is inherently submodular; the incremental amount of information gained along a path decreases due to redundant observations. In addition to submodularity, the incremental amount of information gained is a function of not only the current state but also the entire history as well. This paper presents the construction of the first-order necessary optimality conditions for memory (history-dependent) Lagrangians. Path-dependent problems frequently appear in robotics and artificial intelligence, where the state such as a map is partially observable, and information can only be obtained along a trajectory by local sensing. Robotic exploration and environmental monitoring has numerous real-world applications and can be formulated using the proposed approach.
  
  
\end{abstract}



\section{Introduction}
Drawing a map requires exploration. The quality of the map dependents on the amount of information extracted along the explored path. Therefore, we can construct a function that takes a path as an input and outputs the information learned. Then to construct the best possible map, we find the path that extremizes this function. In fact, this approach is commonly applied to robotic exploration and mapping problems~\cite{binney2012branch,GhaffariJadidi2018a,GhaffariJadidi2019,singh2009efficient}.

It is tempting to apply standard calculus of variations to solve this problem where the Lagrangian is the incremental amount of information to be learned. However, there is a fundamental problem with this approach. Each time a location is explored, the amount of information gathered diminishes. Therefore, \textit{the incremental amount of information to be learned depends on the entire previous history rather than merely the current state}. With this observation, the functionals that we will be considering will have the form
\begin{equation}\label{eq:origonal_problem}
	\mathcal{I}_R[\gamma] = 
	\int_0^T \, L(\tau,\gamma(\tau),\dot{\gamma}(\tau)) \, d\tau - G\left(\int_0^T\, \alpha(\tau,\gamma(\tau))\, d\tau \right).
\end{equation}
The first term corresponds to a cost that does depend only on the current state (in our application to exploration, this term corresponds to the risk associated with the path). The non-linearity of the function $G:\mathbb{R}\to\mathbb{R}$ encodes the dependency on the path's history.

To extremize \eqref{eq:origonal_problem}, we use the indirect method and construct first-order necessary optimality conditions. Under sufficient regularity of the data ($L$, $G$, and $\alpha$), the conditions have the following form.
\begin{theorem}[Main Result]\label{thm:main}
	Let $\mathcal{I}_R:C^0([0,T],M)\to\mathbb{R}$ (where $M$ is a smooth, finite-dimensional manifold) have the form \eqref{eq:origonal_problem}. Then, first-order necessary optimility conditions are given by the following second-order ordinary differential equation
	\begin{equation}\label{eq:MEL}
		\frac{d}{dt}\frac{\partial L}{\partial \dot{q}} - \frac{\partial L}{\partial q} = -\mathcal{K}\cdot \frac{\partial\alpha}{\partial q}, \quad
		\mathcal{K} = G'\left( \int_0^T\, \alpha(\tau,\gamma(\tau))\, d\tau \right).
	\end{equation}
\end{theorem}
The equation \eqref{eq:MEL} is called the \textit{Memory Euler-Lagrange equation} (MEL). An important caveat in this equation is that the constant $\mathcal{K}$ is \textit{unknown}. Although the equations are expressed in closed-form, a shooting problem (or some other technique) is still required to determine the value of $\mathcal{K}$. Some other works related to this subject are \cite{Basin2008,cannarsa2013,Ferialdi_2012,paifelman2019,8623241}. In particular, this work was inspired by \cite{delayproblem}.

The remaining of this paper is organized as follows. Submodular functions are discussed in  section \ref{sec:submodular}. These functions have a diminishing return property which encapsulates the philosophy behind history-dependent Lagrangians for exploration. The formal problem statement (which theorem \ref{thm:main} addresses) is presented in section \ref{sec:problem}. The necessary functional analysis required to prove theorem \ref{thm:main} is reviewed in section \ref{sec:functional}. The main result is proved in section \ref{sec:MEL} where is it compared to the classical Euler-Lagrange equations and it is shown that the principle of optimality fails in our case. A connection to the problem of exploration is shown in section \ref{sec:exploration}. An interpretation of this problem in the Hamiltonian formulation is presented in section \ref{sec:potential}. An analytic example (which can almost be solved in closed-form) is shown in section \ref{sec:analytic} and a numerical example is shown in section \ref{sec:numerical}. Finally, conclusions and future directions are discussed in section \ref{sec:conclusions}.

Optimizing \eqref{eq:origonal_problem} can also be approached using Pontryagin's maximal principle by defining a new state $\dot{y} = \alpha(\tau,\gamma(\tau))$ and imposing a final cost of $G(y(T))$. We will not use this approach as we wish to view $G(y(t))$ as a running cost over each incremental time-step rather than as a final cost. Although optimization problems of the form \eqref{eq:origonal_problem}, our approach (cf. section \ref{sec:MEL}) can handle more exotic Lagrangians which we expect cannot be solved via the maximal principle. This is an object of future work.
\section{Submodular Functions}\label{sec:submodular}
A function is submodular if it has diminishing returns. This section will define this notion and demonstrate how \eqref{eq:origonal_problem} falls under this category when $G$ is nonlinear.

Let $M$ be the environment where we wish to explore. In our case, $M$ will be assumed to be a smooth, finite-dimensional manifold. Let $C_P(M)$ be the set of all continuous paths in $M$:
\begin{equation*}
	\left[ \gamma:[a,b]\to M\right] \in C_P(M).
\end{equation*}
The information gathered is a real-valued function on this path space:
\begin{equation}\label{eq:information}
	\mathcal{I}:C_P(M)\to\mathbb{R},
\end{equation}
where $\mathcal{I}[\gamma]$ is the information gathered from traversing the path $\gamma$. 

We can view $C_P(M)$ as a ``semigroup'' under concatenation (not every element can be combined; they must have common endpoints for their concatenation to be continuous).  Let $\gamma:[0,a]\to M$ and $\zeta:[0,b]\to M$ be two paths. Then their concatenation is $\gamma\frown\zeta:[0,a+b]\to M$ where
\begin{equation}
	\gamma\frown\zeta(t) = \begin{cases}
		\gamma(t), & t\in [0,a] \\
		\zeta(t-a), & t\in [a,a+b].
		\end{cases}
\end{equation}
Concatenation allows us to define submodular functions.
\begin{definition}
The information function $\mathcal{I}:C_P(M)\to\mathbb{R}$ is \textit{submodular} if 
\begin{equation*}
	\mathcal{I}\left[ \gamma\frown\zeta\right] \leq \mathcal{I}[\gamma] + \mathcal{I}[\zeta],
\end{equation*}
where ever it is defined.
\end{definition}
\begin{remark}
	If we use the ``classical'' form of a Lagrangian (whose extremals will be described by the usual Euler-Lagrange equations), which manifests as
	\begin{equation}\label{eq:classical_Lagrangian}
		\mathcal{I}[\gamma] = \int_0^a \, \mathcal{L}(\gamma(t),\dot{\gamma}(t)) \, dt,
	\end{equation}
	then is will have the property that $\mathcal{I}[\gamma\frown\zeta] = \mathcal{I}[\gamma] + \mathcal{I}[\zeta]$, i.e., it will be \textit{modular} and will \textit{not} have the diminishing return property.
\end{remark}
\begin{figure}
    \centering
		\begin{tikzpicture}[scale=0.7]
			\draw[black] plot [smooth cycle] coordinates{ (6.7200,0) (6.4042,1.4097) (5.5333,2.5600) (4.3095,3.2760) (2.9894,3.5194) (1.7968,3.3891) (0.8528,3.0714) (0.1496,2.7600) (-0.4228,2.5791) (-1.0118,2.5393) (-1.7242,2.5430) (-2.5705,2.4349) (-3.4520,2.0770) (-4.1958,1.4137) (-4.6236,0.5029) (-4.6236,-0.5029) (-4.1958,-1.4137) (-3.4520,-2.0770) (-2.5705,-2.4349) (-1.7242,-2.5430) (-1.0118,-2.5393) (-0.4228,-2.5791) (0.1496,-2.7600) (0.8528,-3.0714) (1.7968,-3.3891) (2.9894,-3.5194) (4.3095,-3.2760) (5.5333,-2.5600) (6.4042,-1.4097) };
			\draw[red,very thick] plot [smooth] coordinates{
				(-3,-2) (-2,2) (5,-2) (4.15,0.5) (4,1) (6,0.5)};
			\draw[blue,very thick] plot [smooth] coordinates{
			(-3,-2) (-2,1) (4,3) (2,-2) (5,-2) (4.15,0.5) (4,1) (6,0.5)};
			\begin{scope}
				\clip (3.87,2) rectangle (6,0.5);
				\draw[green,very thick] plot [smooth] coordinates{
					(-3,-2) (-2,2) (5,-2) (4.15,0.5) (4,1) (6,0.5)};
			\end{scope}
			\node at (-0.5,0.5) {\textcolor{red}{\Large $\gamma$}};
			\node at (5,1.5) {\textcolor{green}{\Large $\zeta$}};
			\node at (1.5,2) {\textcolor{blue}{\Large $\xi$}};
			\draw [fill] (6,0.5) circle [radius=0.1];
			\draw [fill] (-3,-2) circle [radius=0.1];
		\end{tikzpicture}
		\caption{The information function, $\mathcal{I}$ is both submodular and path-dependent. Submodular: $\mathcal{I}[\gamma\frown\zeta]\leq \mathcal{I}[\gamma] + \mathcal{I}[\zeta]$. Path-dependent: $\mathcal{I}[\gamma\frown\zeta] - \mathcal{I}[\gamma] \ne \mathcal{I}[\xi\frown\zeta] - \mathcal{I}[\xi]$.}
		\label{fig:path-dependent}
	\end{figure}
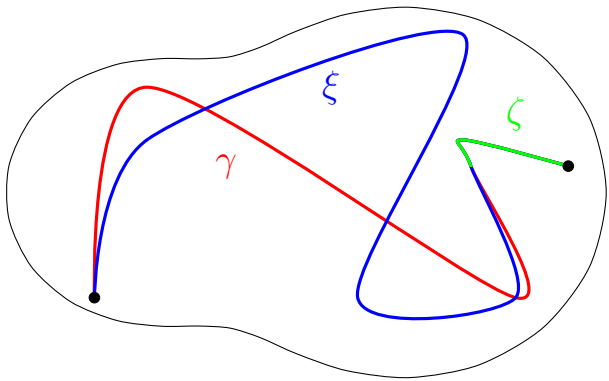
\subsection{``Path integral'' form}
Our primary goal is to utilize a modified version of the calculus of variations to find extrema of \eqref{eq:origonal_problem}. Our first step is to ``differentiate'' $\mathcal{I}$ to make the problem incremental.
\begin{definition}\label{def:diff_path}
	Let $\gamma:[0,a]\to M$ be a continuous path. The function $\mathcal{I}:C_P(M)\to\mathbb{R}$ is differentiable at $\gamma\in C_P(M)$ if, for any two smooth paths $\zeta:[0,b]\to M$ and $\xi:[0,c]\to M$ such that $ \zeta(0) = \xi(0) = \gamma(a)$ and $v = \dot{\zeta}(0) = \dot{\xi}(0)$, we have
	\begin{equation*}
		\lim_{\varepsilon\to 0^+}\, \frac{1}{\varepsilon} \left( \mathcal{I}[\gamma\frown\zeta_\varepsilon] - \mathcal{I}[\gamma]\right) = \lim_{\varepsilon\to 0^+} \, \frac{1}{\varepsilon} \left( \mathcal{I}[\gamma\frown\xi_\varepsilon] - \mathcal{I}[\gamma]\right),
	\end{equation*}
	where $\zeta_\varepsilon:[0,\varepsilon]\to M$ is the restriction (resp. for $\xi_\varepsilon$). The derivative will be denoted as
	\begin{equation}
		\delta\mathcal{I}(\gamma;v) := \lim_{\varepsilon\to 0^+} \, \frac{1}{\varepsilon} \left( \mathcal{I}[\gamma\frown\zeta_\varepsilon] - \mathcal{I}[\gamma]\right).
	\end{equation}
\end{definition}
We will show in the next section that $\mathcal{I}$ in the form given by \eqref{eq:origonal_problem} will be differentiable in the sense of definition \ref{def:diff_path}.

The submodularity of the information can be interpreted in this infinitesimal view-point as:
\begin{equation*}
	\delta\mathcal{I}(\gamma\frown\zeta;v) \leq \delta\mathcal{I}(\zeta;v).
\end{equation*}
A benefit of differentiating the information is that it allows us to write $\mathcal{I}$ reminiscent of \eqref{eq:classical_Lagrangian}:
\begin{equation}
	\mathcal{I}[\gamma] = \int_0^a\, \delta\mathcal{I}(\gamma_s,\dot{\gamma}(s)) \, ds,
\end{equation}
where $\gamma_s:[0,s]\to M$ is the restriction of the path. It is important to stress that $\delta\mathcal{I}(\gamma_s;\dot{\gamma}(s))$ does not only depend on the current state at time $s$, rather the entire path leading up to this instance (recall figure \ref{fig:path-dependent}). Let us call $\mathcal{L}:\mathbb{R}\times C_P(M)\times TM\to \mathbb{R}$ the ``memory Lagrangian'' where
\begin{equation*}
	\mathcal{L}(t,\gamma,v) = \delta\mathcal{I}(\gamma_t;v).
\end{equation*}
This allows us to write our optimization problem as
\begin{equation*}
	\mathcal{I}[\gamma] = \int_0^T \, \mathcal{L}(t,\gamma_t,\dot{\gamma}(t))\, dt.
\end{equation*}
\section{Problem Statement}\label{sec:problem}
As we wish to maximize the information gathered along a path, we want to solve the following maximization problem.
\begin{problem}[Fixed-end point]\label{prob:orig}
	Let $\mathcal{L}:\mathbb{R}\times C_P(M)\times TM\to\mathbb{R}$. Then we want to find
	\begin{equation}\label{eq:problem}
		\arg\max_{\gamma\in C_P(M)} \, \int_0^T \, \mathcal{L}(s,\gamma_s,\dot{\gamma}(s))\, ds,
	\end{equation}
	such that $\gamma(0) = x_0$ and $\gamma(T) = x_f$.
\end{problem}
Our goal is to develop first-order necessary conditions for an extremal (which will be reminiscent of the Euler-Lagrange equations).
\subsection{A particular class of Lagrangians}
Suppose that our (submodular) function has the form
\begin{equation}\label{eq:submodular_form}
	\mathcal{I}[\gamma] = G\left( \int_0^T\, \alpha(\tau,\gamma(\tau))\, d\tau \right).
\end{equation}
This function is submodular precisely when (a) $G$ is concave down and (b) $\alpha:\mathbb{R}\times M\to\mathbb{R}$ is positive (or non-negative). Differentiating this via definition \ref{def:diff_path}, we get that
\begin{equation}\label{eq:partial_Lagrangian}
	\delta\mathcal{I}(\gamma_s;v) = \alpha(s,\gamma(s)) \cdot G'\left( \int_0^s\, \alpha(\tau,\gamma(\tau))\, d\tau \right).
\end{equation}
It is important to point out that \eqref{eq:partial_Lagrangian} does not actually depend on $v$. Comparing this to classical Lagrangians, \eqref{eq:partial_Lagrangian} has the form of a ``potential.'' To make this problem dynamic, we must include a dependence on $v$; in practice this will manifest as a ``risk measure.'' Let $L:\mathbb{R}\times TM\to\mathbb{R}$ be the risk associated to the dynamic state in $TM$ at time $t\in\mathbb{R}$. As we wish to maximize information while minimizing risk, we wish to extremize their \textit{difference}. Putting this together, the memory Lagrangian will have the form
\begin{equation}\label{eq:ML_form}
	\boxed{
	\mathcal{L}(s,\gamma_s,\dot{\gamma}(s)) = L(s,\gamma(s),\dot{\gamma}(s)) - \alpha(s,\gamma(s)) \cdot \beta\left( \int_0^s\, \alpha(\tau,\gamma(\tau)) \, d\tau \right),}
\end{equation}
where $\beta = G'$ and all the data is assumed to be smooth. In particular, this is the memory Lagrangian corresponding to \eqref{eq:origonal_problem}.
\section{Some Functional Analysis}\label{sec:functional}
Taking variations of a path-dependent Lagrangian is a more delicate procedure than in the classical case. This is due to the fact that $\mathcal{L}$ is a function on an infinite-dimensional space which requires tools from functional analysis. This section reviews some results from functional analysis and, in particular, derivatives of functions whose domain is the space of continuous functions. Much of the information covered in this section on functional analysis can be found in the books \cite{debnath2005introduction} and \cite{hunter2001applied}.

In this section, we will assume that $M\subset V$ of some finite-dimensional vector space. This will not be of too much concern because differentiation is local and $M$ can be locally viewed as a subset of some vector space.

Suppose that we have a fixed time interval, $[0,T]\subset\mathbb{R}$, and let $V$ be a finite-dimensional vector space. Then the set of continuous functions
\begin{equation}\label{eq:path_space}
	C^0\left([0,T],V\right) = \left\{ \gamma:[0,T]\to V \right\},
\end{equation}
forms an infinite-dimensional space and is a Banach space with the supremum norm.
\subsection{Banach Spaces}
We first recall the definition of a Banach space.
\begin{definition}
	A Banach space is a normed linear space that is a complete metric space with respect to the metric derived from its norm.
\end{definition}
\begin{example}
	The space $C^0([a,b],\mathbb{R})$ of continuous real-valued functions on the interval $[a,b]$ with the sup-norm forms a Banach space,
	\begin{equation*}
		\lVert f\rVert_\infty = \sup_{x\in[a,b]} \, |f(x)|.
	\end{equation*}
	As a consequence, our space of interest \eqref{eq:path_space} will be a Banach space.
\end{example}
\subsection{Differentiability}
Let $X$ and $Y$ be Banach spaces and $f:X\to Y$ a continuous function.
Extremal values of $f$ will be given by critical points which will require differentiation.
There are multiple definitions for differentiability in Banach spaces; the two that we will focus on are the stronger Fr\'{e}chet and the weaker G\^{a}teaux derivative.
\begin{definition}
	A map $f:X\to Y$ is differentiable at $x\in X$ if there is a bounded linear map $A:X\to Y$ such that
	\begin{equation*}
		\lim_{h\to 0} \, \frac{\lVert f(x+h)-f(x)-Ah\rVert}{\lVert h\rVert}.
	\end{equation*}
\end{definition}
This notion of differentiability is called the \textit{Fr\'{e}chet derivative}. In practice, to solve \eqref{eq:problem}, we will work with the directional, or G\^{a}teaux, derivative (see definition \ref{def:gateaux} below).

If the map $f:X\to Y$ is Fr\'{e}chet differentiable, then $f'(x)=A:X\to Y$ is a linear and bounded map. We can think of the derivative as a map
\begin{equation*}
	f':X\to\mathcal{B}(X,Y),
\end{equation*}
where $\mathcal{B}(X,Y)$ is the space of all bounded linear functions. When the codomain is the real numbers (as is the case in \eqref{eq:information}), the derivative has the form
\begin{equation*}
	f':X\to\mathcal{B}(X,\mathbb{R}) = X^*,
\end{equation*}
where $X^*$ is the dual space.
\begin{remark}
	This is reminiscent of the exterior derivative on manifolds in the following way: let $f:M\to\mathbb{R}$ be a smooth function, then its derivative $df:M\to T^*M$ consists of co-vectors. A (Riemannian) metric on $M$ induces an isomorphism $T^*M\cong TM$ which transforms the differential into the gradient. The map $\mathcal{R}$ in \eqref{eq:dual_nbv} will accomplish a similar feat.
\end{remark}
We end our discussion of differentiability with the idea of the directional, or G\^{a}teaux, derivative.
\begin{definition}\label{def:gateaux}
	Let $X$ and $Y$ be Banach spaces and $f:X\to Y$. The directional derivative of $f$ at $x\in X$ in the direction $h\in X$ is given by
	\begin{equation*}
		\delta f(x;h) = \lim_{t\to 0} \frac{f(x+th)-f(x)}{t}.
	\end{equation*}
	If this limit exists for every $h\in X$, and $f'_G(x):X\to Y$ defined by $f'_G(x)h = \delta f(x;h)$ is a linear map, then we say that $f$ is G\^{a}teaux differentiable at $x$ and $f'_G$ the G\^{a}teaux derivative.
\end{definition}
Another way of writing the G\^{a}teaux derivative is
\begin{equation}\label{eq:directional_derivative}
	\delta f(x;h) = \left.\frac{d}{dt}\right|_{t=0} \, f(x+th).
\end{equation}
\begin{remark}
	If $f$ is Fr\'{e}chet differentiable, it is also G\^{a}teaux differentiable and the Fr\'{e}chet derivative, $f'(x)$, is given by
	\begin{equation*}
		f'(x)h = \delta f(x;h).
	\end{equation*}
	That is, $f'(x) = f'_G(x)$. However, the converse is \textit{not} generally true. G\^{a}teaux differentiability does not imply Fr\'{e}chet differentiability. As such, when dealing with differentiability, we will assume Fr\'{e}chet but will commonly compute via \eqref{eq:directional_derivative}.
\end{remark}
\subsection{Riemann-Stieltjes integration}
Our information functions has a form similar to $f:C^0([a,b],\mathbb{R})\to\mathbb{R}$ and thus its derivative has values in $C^0([a,b],\mathbb{R})^*$. This dual space has close connections to Riemann-Stieltjes integration.
The contents below are extracted, mostly, from \S4.4 in \cite{kreyszig1978introductory}.
\subsubsection{Bounded Variation}
A function $w:[a,b]\to\mathbb{R}$ is said to be of bounded variation if its total variation is finite where
\begin{equation*}
	V_a^b(w) = \sup_{P\in\mathcal{P}} \, \sum_{i=0}^{n_P-1} \, \left| w(t_{i+1}) - w(t_i) \right|,
\end{equation*}
and the supremum runs over all partitions of the interval,
$P = \left\{ a= t_0,\ldots, t_{n_P}=b \right\}$. Call the vector space of all functions with bounded variation by $BV([a,b],\mathbb{R})$. A norm on this space is $\lVert w \rVert_{BV} = \left|w(a)\right| + V_a^b(w)$.

\subsubsection{The Riemann-Stieltjes Integral}
Consider a continuous function $x\in C^0([a,b],\mathbb{R})$ and a function of bounded variation $w\in BV([a,b],\mathbb{R})$. For a partition, $P = \left\{t_0,\ldots,t_n\right\}$, let $\eta(P)$ be the largest interval:
\begin{equation*}
	\eta(P) = \max \left\{ t_1-t_0, \ldots, t_{n_P}-t_{n_P-1} \right\}.
\end{equation*}
For each partition, consider the finite sum
\begin{equation*}
	s(P) = \sum_{i=1}^{n} \, x(t_i) \left[ w(t_i)-w(t_{i-1}) \right].
\end{equation*}
If for every $\varepsilon>0$, there exists $\delta>0$ such that if
	$\eta({P}) < \delta$ then $\left| \mathcal{J} - s(P)\right|<\varepsilon$,
then $\mathcal{J}$ is the Riemann-Stieltjes integral and is written via
\begin{equation}
	\mathcal{J} = \int_a^b \, x(t) \, dw(t).
\end{equation}
The Riemann-Stieltjes integral allows us to associate the dual space to continuous functions with functions of bounded variation.

\begin{theorem}[Riesz's theorem]\label{th:riesz}
	Every bounded linear functional can be represented by a Riemann-Stieltjes integral:
	\begin{equation}
		f\in C^0([a,b],\mathbb{R})^* \implies f(x) = \int_a^b \, x(t) \, dw(t).
	\end{equation}
\end{theorem}
Theorem \ref{th:riesz} shows that there is an intimate relationship between $C^0([a,b],\mathbb{R})^*$ and $BV([a,b],\mathbb{R})$. However, the relationship is not unique; if $\tilde{w} = w + c$ where $c$ is a constant, they induce the same integral. Therefore, we define the set of \textit{normalized} bounded variation functions,
\begin{equation*}
	NBV([a,b],\mathbb{R}) = \left\{ g\in BV([a,b],\mathbb{R}) : g(a)=0,~ 
	\lim_{x\to c^-}f(x) = f(c) \right\}.
\end{equation*}
With this normalization, there exists a 1-1 correspondence
\begin{equation*}
	C^0([a,b],\mathbb{R})^* \longleftrightarrow NBV([a,b],\mathbb{R}).
\end{equation*}
Let $\mathcal{R}:C^0([a,b],\mathbb{R})^*\to NBV([a,b],\mathbb{R})$ be this association, i.e.
\begin{equation}\label{eq:dual_nbv}
	f(x) = \int_a^b \, x(t) \, d\left[ \mathcal{R}(f)\right](t).
\end{equation}
It is shown in \cite{delayproblem} that $\mathcal{R}$ is a linear topological isomorphism.
\subsection{Higher Dimensions}
Everything discussed so far has been for the space $C^0([a,b],\mathbb{R})$, i.e. the codomain is a 1-dimensional vector space. Let $V$ be a finite-dimensional vector space (e.g. $V=\mathbb{R}^n$). Everything previously discussed still works for this case. Let $\lVert\cdot\rVert_V$ be a norm on $V$, then for a function $w:[a,b]\to V$,
\begin{equation*}
	V_a^b(w) = \sup_{P\in\mathcal{P}} \, \sum_{i=0}^{n_P-1} \, \lVert w(t_{i+1})-w(t_{i}) \rVert_V.
\end{equation*}
Moreover, the Riemann-Stieltjes integral follows as usual where $w\in BV([a,b],V^*)$,
\begin{equation*}
	s\left( P \right) = \sum_{i=1}^{n} \langle x(t_i), \left[ w(t_i)-w(t_{i-1}) \right] \rangle.
\end{equation*}
Riesz's theorem in this context states that
\begin{equation*}
	C^0([a,b],V)^* \longleftrightarrow NBV([a,b],V^*).
\end{equation*}
Therefore, for a Fr\'{e}chet differentiable function $f:C^0([a,b],V)\to\mathbb{R}$, its derivative can be represented as a normalized function of bounded variation, $w:[a,b]\to V^*$. 

We will call this isomorphism $\mathcal{R}_V:C^0([a,b],V)^*\to NBV([a,b],V^*)$.
\section{The Memory Euler-Lagrange Equations}\label{sec:MEL}
We now return to the problem of finding solutions to \eqref{eq:problem}. Recall that the memory Lagrangian has the form $\mathcal{L}:[0,T]\times C^0([0,T],M) \times TM\to\mathbb{R}$. To solve this, we make three regularity assumptions on $\mathcal{L}$, \cite{delayproblem}.
\begin{assumption}\label{ass:regularity}
	We make the following three regularity assumptions.
	\begin{itemize}
		\item[(A.1)] $\mathcal{L}$ is continuous, i.e. $\mathcal{L}\in C^0\left( [0,T]\times C^0([0,T],M)\times TM,\mathbb{R}\right)$. 
		\item[(A.2)] For all $(t,\gamma,v)\in [0,T]\times C^0([0,T],M)\times TM$, the partial Fr\'{e}chet derivative exists with respect to the second (path) variable and is continuous.
		\item[(A.3)] For all $(t,\gamma,v)\in [0,T]\times C^0([0,T],M)\times TM$, the partial Fr\'{e}chet derivative exists with respect to the third (vector) variable and is continuous.
	\end{itemize}
\end{assumption}
\begin{remark}
	This assumption will not be restrictive as Lagrangians of the form \eqref{eq:ML_form} satisfy assumption \ref{ass:regularity}.
\end{remark}
To find critical paths of $\mathcal{I}$, we take variations. Let $h$ be a variation along the path $\gamma$, i.e. $h(t) \in T_{\gamma(t)}M$. Then,
\begin{equation*}
	\begin{split}
	\delta\int_0^T\, \mathcal{L}(s,\gamma_s,\dot{\gamma}(s))\, ds &= 
	\int_0^T\, \left[ \frac{\partial\mathcal{L}}{\partial \gamma}\cdot h_s + \frac{\partial \mathcal{L}}{\partial v} \cdot \dot{h}(s) \right] \, ds \\
	&= \int_0^T \, \left[ \int_0^s\, h(\theta)\cdot d\left[ \mathcal{R}_V\left(\frac{\partial \mathcal{L}}{\partial \gamma}\right) \right](\theta) - \frac{d}{dt}\frac{\partial\mathcal{L}}{\partial v}\cdot h(s) \right] \, ds,
	\end{split}
\end{equation*}
where we used that $h(0)=0$ and $h(T)=0$ (as problem \ref{prob:orig} has fixed end points). An issue we are faced with is how to extract $h(s)$ out of \textit{both} terms. The following lemma and proposition demonstrate how to accomplish this.
\begin{lemma}
	Let $h:[0,T]\to\mathbb{R}$ and $k:[0,T]\times [0,T]\to\mathbb{R}$ be continuous. Then
	\begin{equation*}
		\int_0^T\, \int_0^t \, h(\theta)\cdot k(t,\theta)\, d\theta \, dt = 
		\int_0^T \, h(t) \cdot \int_t^T\, k(s,t) \, ds \, dt.
	\end{equation*}
\end{lemma}
\begin{proof}
	This follows from Fubini's theorem.
\end{proof}
\begin{proposition}
	Let $g:D\to T^*M$ where $D = \{(t,\theta): 0\leq\theta\leq t \leq T \}$ and
	\begin{equation*}
		\frac{\partial \mathcal{L}}{\partial \gamma}(t,\gamma_t,v)\cdot h_t = \int_0^t \, h_t(\theta) \cdot dg_t(\theta), \quad g_t(\theta) = g(t,\theta).
	\end{equation*}
	Then, assuming that $g$ is differentiable with respect to $\theta$, we have
	\begin{equation}\label{eq:EL_fubini}
		\int_0^T\, \frac{\partial\mathcal{L}}{\partial\gamma}(t,\gamma_t,v) \cdot h_t \, dt = 
		\int_0^T \, \left( \int_t^T \, \frac{\partial}{\partial\theta}g(s,t) \right) \cdot h(t) \, dt.
	\end{equation}
\end{proposition}
For more details on when $g$ is not differentiable, see \cite{delayproblem}. In particular, when $g$ is not differentiable, the right hand side of \eqref{eq:EL_fubini} can be rewritten in integral form as
\begin{equation*}
	\int_0^T\, \frac{\partial\mathcal{L}}{\partial\gamma} (t,\gamma_t,v)\cdot h_t \, dt = 
	\int_0^T \, \left( g(t,t) + \frac{d}{dt}\int_t^T \, g(s,t)\, ds \right) \cdot h(t) \, dt.
\end{equation*}
We can now combine the above to produce the general memory Euler-Lagrange equations.
\begin{theorem}[General memory Euler-Lagrange equations]
	First-order necessary conditions for an extremal of \eqref{eq:problem} are
	\begin{equation}\label{eq:mem_EL}
		\boxed{
		\frac{d}{dt}\frac{\partial\mathcal{L}}{\partial v} = g(t,t) + \frac{d}{dt}\int_t^T \, g(s,t)\, ds, \quad g(t,\theta) = \mathcal{R}_V\left[ \frac{\partial\mathcal{L}}{\partial\gamma}(t,\gamma_t,v) \right] (\theta).}
	\end{equation}
\end{theorem}
\begin{remark}
	For the function $g(t,\theta)$, $t$ represents the current length of the path segment while $\theta\in [0,t]$ is the location of the perturbation along this segment. Additionally, the ``potential'' term in \eqref{eq:mem_EL} involves the tail of the trajectory (it depends on the future of the path). This breaks causality.
\end{remark}
\subsection{Comparison to the ``classical'' E-L equations}
The equation \eqref{eq:mem_EL} is a generalization of the usual Euler-Lagrange equations. Therefore, when the Lagrangian no longer depends on the past, the general memory Euler-Lagrange equations should reduce the the classical Euler-Lagrange equations.
Suppose that $\mathcal{L}$ only depends on the current position rather than the entire history. Then we can write it as
\begin{equation*}
	\mathcal{L}(t,\gamma_t,\dot{\gamma}(t)) = L(t,\gamma(t),\dot{\gamma}(t))
\end{equation*}
for $L:\mathbb{R}\times M\times V\to \mathbb{R}$. In this situation, the (density of the) derivative with respect to $\gamma$ is
\begin{equation*}
	\frac{\partial\mathcal{L}}{\partial\gamma}(t,\gamma_t,\dot{\gamma}(t))\cdot h_t = 
	\frac{\partial L}{\partial{q}}\left(t,\gamma(t),\dot{\gamma}(t)\right)\cdot h(t).
\end{equation*}
Referring to the right hand side of \eqref{eq:EL_fubini} in a ``distributional sense,'' we have
\begin{equation*}
	\frac{\partial}{\partial\theta}g(t,\theta) = \frac{\partial L}{\partial{q}}\cdot \delta(t-\theta),
\end{equation*}
which provides us with
\begin{equation*}
	\frac{d}{dt}\frac{\partial L}{\partial\dot{q}} = \frac{\partial L}{\partial q},\quad \left(\frac{\partial\mathcal{L}}{\partial\dot{q}} = \frac{\partial L}{\partial\dot{q}}\right)
\end{equation*}
which is precisely the classical Euler-Lagrange equation.
\subsection{The MEL equations for our class of Lagrangians}\label{sec:ourMEL}
For an arbitrary path-dependent Lagrangian, it is not at all straight-forward to compute the function $g(t,\theta)$. As such, we will compute this for the general class of Lagrangians given by \eqref{eq:ML_form}. 

We will start with the path variation computation (and ignoring the $L$ term as this will only contribute a classical Euler-Lagrange term). For simplicity of calculations, let us again assume that $M\subset V$ is a subset of a finite-dimensional vector space. Let $k:[0,T]\to V$ be a variation. Then we have
\begin{equation*}
	\begin{split}
		\frac{\partial\mathcal{L}}{\partial\gamma}\cdot k_t &= \left.\frac{d}{d\tau}\right|_{\tau=0} \,-\alpha(t,x(t) + \tau\cdot k(t)) \cdot \beta\left( \int_0^t \, \alpha(s,x(s)+\tau\cdot k(s)) \, ds \right) \\
		&= -\left\langle\frac{\partial\alpha}{\partial x}(t,x(t)), k(t)\right\rangle \cdot \beta\left(\int_0^t \, \alpha(s,x(s)) \, ds \right)\\
		&\quad - \alpha(t,x(t)) \cdot \int_0^t \, \left\langle \frac{\partial\alpha}{\partial x}(s,x(s)), k(s)\right\rangle\, ds \cdot \beta'\left( \int_0^t \, \alpha(s,x(s))\, ds \right).
	\end{split}
\end{equation*}
Extracting out the ``density,'' we have (viewing this in a distributional sense)
\begin{equation*}
	\begin{split}
		dg_t(s) &= -\frac{\partial\alpha}{\partial x}(t,x(t)) \cdot \beta\left(\int_0^t \, \alpha(\tau,x(\tau))\, d\tau \right) \cdot \delta(t-s) \\
		&\quad - \alpha(t,x(t)) \cdot \frac{\partial\alpha}{\partial x}(s,x(s)) \cdot \beta'\left(\int_0^t \, \alpha(\tau,x(\tau)) \, d\tau \right).
	\end{split}
\end{equation*}
The memory Euler-Lagrange equations are then
\begin{equation*}
	\begin{split}
		\frac{d}{dt}\frac{\partial L}{\partial v} - \frac{\partial L}{\partial x} &=
		-\frac{\partial\alpha}{\partial x} (t,x(t)) \cdot \beta\left(\int_0^t \, \alpha(\tau,x(\tau)) \, d\tau \right) \\
		&\quad - \int_t^T \, \alpha(s,x(s))\cdot \frac{\partial\alpha}{\partial x}(t,x(t)) \cdot \beta' \left(\int_0^s \, \alpha(\tau,x(\tau))\, d\tau \right) \, ds
	\end{split}
\end{equation*}
The integral over $[t,T]$ above can actually be integrated. This yields
\begin{equation}\label{eq:noncausal_MEL}
	\begin{split}
		\frac{d}{dt}\frac{\partial L}{\partial v} - \frac{\partial L}{\partial x} &=
		-\frac{\partial\alpha}{\partial x}(t,x(t))\cdot \beta\left(\int_0^T\, \alpha(\tau,x(\tau))\, d\tau \right).
	\end{split}
\end{equation}
Notice that although \eqref{eq:noncausal_MEL} depends on future information, it enters via a constant:
\begin{equation}\label{eq:MEL!}
	\boxed{
		\frac{d}{dt}\frac{\partial L}{\partial v} - \frac{\partial L}{\partial x} = 
		-\mathcal{K}\cdot \frac{\partial\alpha_t}{\partial x} }
\end{equation}
where $\mathcal{K}$ is a constant given by
\begin{equation*}
	\mathcal{K} := \beta\left(\int_0^T\, \alpha(\tau,x(\tau))\, d\tau \right).
\end{equation*}
Formally, the constant in \eqref{eq:MEL!} depends on the entire trajectory. However, we can treat it as an unknown parameter to be maximized over. This can be accomplished in the following five step program:
\begin{enumerate}
	\item Prescribe fixed boundary conditions: $x(0)$, $x(T)$, and $T$.
	\item For a given $\mathcal{K}$, solve the shooting problem \eqref{eq:MEL!}.
	\item Along the trajectory found by shooting, calculate $\mathcal{I}_R[\gamma]$ in \eqref{eq:origonal_problem}.
	\item Construct a function, $\mathcal{M}:[\min \beta,\max\beta]\to\mathbb{R}$, given by $\mathcal{K}\mapsto \mathcal{I}_R[\gamma]$.
	\item Perform optimization over the set $[\min\beta,\max\beta]$.
\end{enumerate}
\begin{remark}
	If we choose $\beta$ such that it is bounded from above and below, the optimization to determine $\mathcal{K}$ is done over a compact interval, which allows for compact optimization.
\end{remark}
\subsection{Remark on the principle of optimality}
We note that the principle of optimality \textbf{does not hold} (unless $\beta$ is constant or $\alpha=0$ in which case $G$ is linear). Let $x^*(t)$ be an optimal trajectory and choose $t^*\in (0,T)$. Reconsider the optimization problem on the interval $[t^*,T]$ with the boundary conditions $x(t^*)=x^*(t^*)$ and $x(T)=x^*(T)$. Then the MEL equations will have the same form but the constant will change,
\begin{equation*}
	\mathcal{K}_{t^*} = \beta\left(\int_{t^*}^T\, \alpha(\tau,x(\tau))\, d\tau\right),
\end{equation*}
which is \textit{not} constant as $t^*$ changes. Therefore, the tails of an optimal trajectory are \textit{not} optimal and the principle of optimality fails. This means that dynamic programming and Hamilton-Jacobi-Bellman methods will \textbf{not work}.
\section{Application to Exploration}\label{sec:exploration}
We will consider a purely toy example to model information gathered along a path. As a simple model of information growth, suppose that we know $x$ amount of information (where $x\in[0,1]$ as a percentage). Then we take the rate of knowledge growth to be given by
\begin{equation*}
	\dot{x}(t) = a(t)\left(1-x(t)\right), \quad x(0)=0.
\end{equation*}
That is, the rate of information gain is the product of the learning rate $a(t)$ and the amount left to be learned, $1-x$.
This can be solved via separation of variables to give
\begin{equation*}
	x(t) = 1 - \exp\left( -\int_0^t \, a(t) \, dt \right).
\end{equation*}
This inspires the following idea: Let $I[y|\gamma]$ be the information learned about the point $y$ from traversing the path $\gamma$. Suppose that the rate of learning is dependent on distance and decays as a Gaussian. Then we have
\begin{equation}
	I[y|\gamma] = 1 - \exp\left( -\lambda\cdot\int_0^T \, \exp\left(-\frac{\lVert \gamma(t)-y\rVert^2}{2\ell^2} \right) \, dt \right),
\end{equation}
where $\lambda,\ell>0$ are parameters that describe the learning rate and the distance cutoff, respectively.

However, this is the information gathered about only a single point. To determine the total amount of information gathered, we will integrate over $y$. Then,
\begin{equation}\label{eq:total_information}
	\mathcal{I}[\gamma] = \int_{M} \, I[y|\gamma] \cdot \nu(y) \, dy,
\end{equation}
where $\nu:M\to\mathbb{R}$ is the amount of information contained at each point.

For the meantime, let us ignore $\mathcal{I}[\gamma]$ and work with $I[y|\gamma]$ to avoid integration difficulties. This can be put in the form \eqref{eq:submodular_form} where
\begin{equation*}
	\begin{split}
		G(z) &= 1 - \exp(-\lambda\cdot z) \implies \beta(z) = \lambda\cdot\exp(-\lambda\cdot z), \\
		\alpha(x) &= \exp\left(-\frac{\lVert x-y\rVert^2}{2\ell^2}\right) =: k(x,y).
	\end{split}
\end{equation*}
To apply \eqref{eq:MEL!}, we need to compute the derivative of $\alpha$,
\begin{equation*}
	\begin{split}
		\frac{\partial\alpha}{\partial x} &= \frac{1}{\ell^2}\left[y-x\right] \cdot k(x,y).
	\end{split}
\end{equation*}
To choose $L$, we will only penalize the velocity, i.e.
\begin{equation*}
	L = \frac{1}{2}m\lVert \dot{x}\rVert^2,
\end{equation*}
where $m>0$ is a parameter (this will act as a regularizing term). The memory Euler-Lagrange equations are then
\begin{equation}\label{eq:example_MEL}
	\boxed{
		m\ddot{x}(t) = \mathcal{K}\cdot \frac{1}{\ell^2}\cdot \left[x-y\right] \cdot k(x,y), \quad
		0\leq\mathcal{K}\leq \lambda.}
\end{equation}
If we return to the complete problem dealing with $\mathcal{I}[\gamma]$ rather than $I[y|\gamma]$, we have
\begin{equation}\label{eq:full_map_dynamics}
	m\ddot{x} = \frac{1}{\ell^2} \cdot \int_M\, \mathcal{K}_y \cdot (x-y)\cdot k(x,y)\cdot \nu(y)\, dy, \quad 0\leq \mathcal{K} \leq \lambda,
\end{equation}
\textit{where we are ignoring any issues with having an infinite amount of memory potential terms}.

A troublesome aspect of \eqref{eq:full_map_dynamics} is that each point $y\in M$ has its own unknown constant, $\mathcal{K}_y$. Therefore, to fully understand the dynamics we have to determine an \textit{unknown function}, $\mathcal{K}_y$.
\subsection{A test example}\label{sec:small_ex}
To test our MEL equations \eqref{eq:example_MEL}, we will compute a simple example where $M = [0,2]$ such that our conditions can be calibrated against brute-force methods.

Suppose that $\nu(y) = \delta(1-y)$ so all information is concentrated at the center of $M$. Then the information function is (taking $\lambda=1$)
\begin{equation*}
	\mathcal{I}[\gamma] = 1 - \exp\left( - \int_0^T \, \exp\left( -\frac{|\gamma(t)-1|^2}{2\ell^2} \right) \, dt \right).
\end{equation*}
Adding in the risk term as described above, we wish to maximize the following:
\begin{equation}\label{eq:to_optimize}
	\mathcal{I}_R[\gamma] = 1 - \exp\left( - \int_0^T \, \exp\left( -\frac{|\gamma(t)-1|^2}{2\ell^2} \right) \, dt \right) - \int_0^T \, \frac{m}{2}\dot{\gamma}(t)^2 \, dt.
\end{equation}
Suppose that we discretize the interval $[0,2]$ into a uniform partition with $N$ pieces. Then \eqref{eq:to_optimize} can be approximated via
\begin{equation}\label{eq:discrete_information}
	\mathcal{I}_R[\gamma] \approx 1- \exp\left( -\sum_{i=2}^N \, \exp\left( -\frac{|x_i-1|^2}{2\ell^2} \right) \Delta t \right)
	- \frac{m}{2\Delta t} \sum_{i=1}^{N-1} \, \left( x_{i+1} - x_i\right)^2.
\end{equation}
On the other hand, the MEL equations state that the optimal path satisfies the ODE
\begin{equation}\label{eq:cont_to_optimize}
	m\ddot{x} = \mathcal{K}\cdot \frac{1}{\ell^2} \cdot \left( x-1 \right) \cdot k(x,1).
\end{equation}
Results from the discrete optimization, \eqref{eq:to_optimize}, along with the MEL solution, \eqref{eq:cont_to_optimize}, are shown in figure \ref{fig:comparing_brute_MEL}.

\begin{figure}
	\centering
	\begin{subfigure}[t]{0.45\textwidth}
		\centering
		\includegraphics[width=\textwidth]{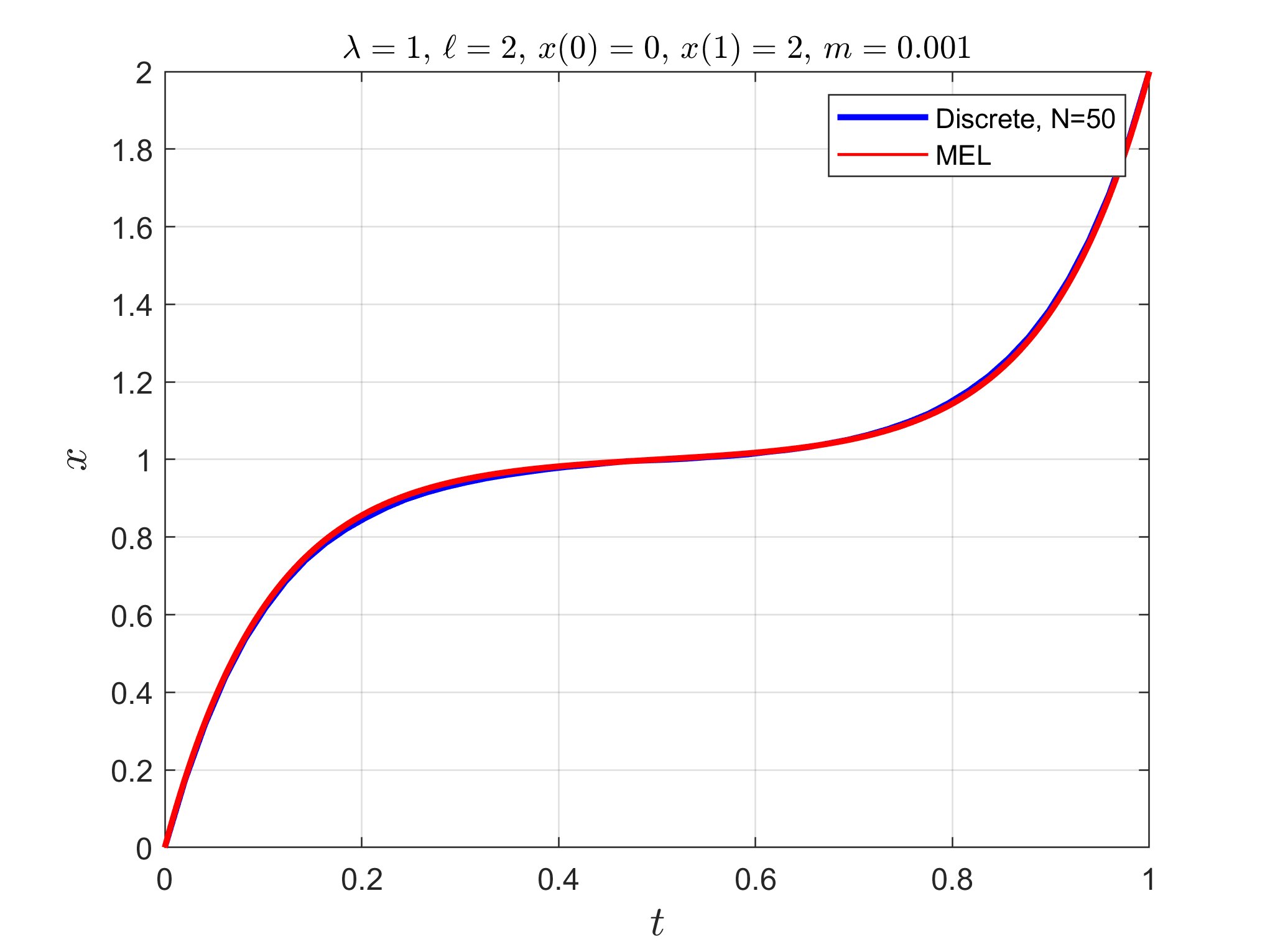}
		\caption{Comparison of the brute-force method with the predicted MEL solution.}
		\label{fig:comparing_brute_MEL}
	\end{subfigure}
	\hfill
	\begin{subfigure}[t]{0.45\textwidth}
		\centering
		\includegraphics[width=\textwidth]{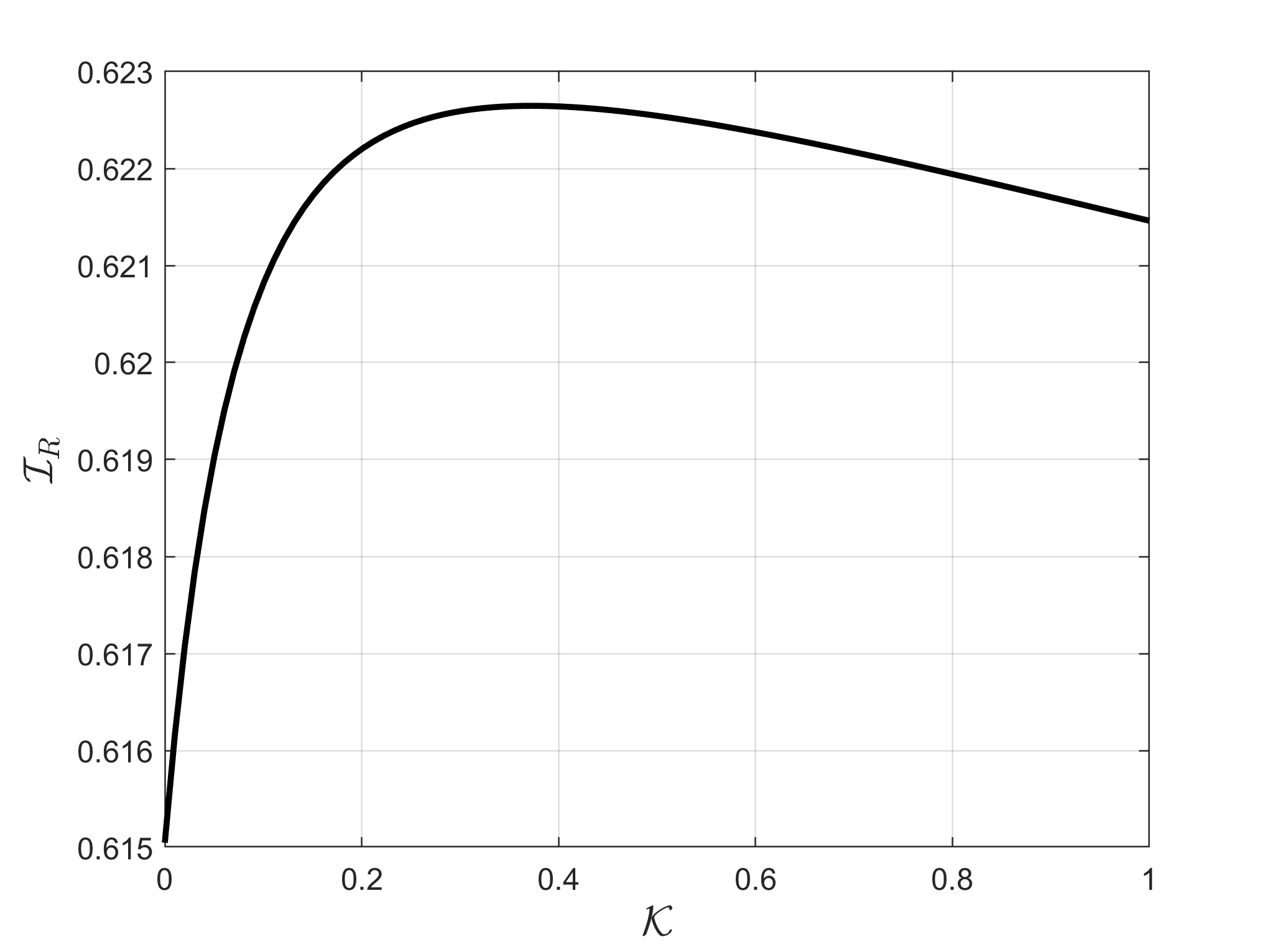}
		\caption{A plot demonstrating how $\mathcal{I}_R$ depends on the parameter $\mathcal{K}$ as described at the end of section \ref{sec:ourMEL}. The optimal $\mathcal{K}$ value is $\mathcal{K}=0.3801$.}
		\label{fig:K_dependency}
	\end{subfigure}
	\caption{Plots depciting the solution to the simple exploration problem laid out in section \ref{sec:small_ex}. The boundary conditions for this example are $x(0)=0$ and $x(1)=2$.}
	\label{fig:two_graphs}
\end{figure}
\section{Information Potential}\label{sec:potential}
This path planning problem, problem \ref{prob:orig}, was originally set up with a Lagrangian and solved by variations. However, this problem can instead be posed in the Hamiltonian framework. Recall that our memory Lagrangians have the form
\begin{equation*}
	\mathcal{L}(t,x_t,\dot{x}(t)) = L(t,x(t),\dot{x}(t)) - \alpha(t,x(t))\cdot \beta\left(\int_0^t\, \alpha(\tau,x(\tau))\, d\tau \right).
\end{equation*}
Suppose that $L$ is hyperregular (so the fiber derivative is a global diffeomorphism) and let $H:T^*M\to\mathbb{R}$ be its Legendre transform. Then the equations of motion, \eqref{eq:MEL!}, are equivalent to 
\begin{equation}\label{eq:forced_Hamiltonian}
	i_X\omega = dH + \mathcal{K}\cdot \pi_M^*d\alpha_t,
\end{equation}
where $\omega$ is the standard symplectic form on $T^*M$, $i_X\omega = \omega(X,\cdot)$ is the interior contraction, and $\pi_M:T^*Q\to Q$ is the cotangent projection. The time derivative of $H$ is
\begin{equation}
	\begin{split}
		\frac{d}{dt}H &= dH(X) = \omega(X_H,X) = \left(-i_X\omega\right)(X_H) \\
		&= -dH(X_H) - \mathcal{K}\cdot\pi_M^*d\alpha(X_H) \\
		&= -\mathcal{K}\cdot d\alpha(v) =: \rho(v).
	\end{split}
\end{equation}
This shows that the ``information potential'' does work on the system. As the potential is conservative ($\rho$ is exact), we can integrate. Define the following function as a modification of the energy,
\begin{equation}
	\mathcal{H}_\mathcal{K} = H + \mathcal{K}\cdot \pi_M^*\alpha,
\end{equation}
which \textit{is} conserved along the trajectory.
\begin{remark}
	The modified energy $\mathcal{H}_\mathcal{K}$ along with \eqref{eq:forced_Hamiltonian} seem to imply that the MEL equations can be taken to be Hamiltonian. However, this is a much more subtle question. This is because $\mathcal{K}$ depends on each trajectory, i.e. it is not a global constant. The actual flow is a conglomeration of a family of different Hamiltonians. It is not at all obvious that these should glue together to be Hamiltonian.
\end{remark}
This section concludes with two subsection. The first is a computation on how $\mathcal{I}_R[\gamma_t]$ changes along the trajectory, i.e. how information is accumulated and how this is related to $\rho$. The second subsection attempts to construct a time-1 map and examine whether or not it is Hamiltonian/symplectic.
\subsection{Information accumulation}
We consider how $\mathcal{I}_R[\gamma_t]$ changes with $t$. Call the functions
\begin{equation*}
	\begin{split}
		A(t) &= G\left( \int_0^t\, \alpha(\tau,\gamma(\tau))\, d\tau \right), \\
		A_R(t) &= G\left( \int_0^t \, \alpha(\tau,\gamma(\tau))\, d\tau \right) - \int_0^t \, L(\tau,\gamma(\tau),\dot{\gamma}(\tau))\, d\tau.
	\end{split}
\end{equation*}
The function $A$ records the total amount of information gathered up to time $t$ while $A_R$ is the risk-adjusted amount of information. We would expect (due to the problem being submodular) that $A$ (or $A_R$) would be concave down. However, this is not necessarily true. Computing derivatives yields:
\begin{equation*}
	\begin{split}
		A'(t) &= \alpha(t,\gamma(t)) \cdot G'\left( \int_0^t\, \alpha(\tau,\gamma(\tau))\, d\tau \right) \\
		A''(t) &= d\alpha_t\left(\dot{\gamma}(t)\right) \cdot G'\left(\int_0^t\, \alpha(\tau,\gamma(\tau))\, d\tau \right) + \\
		&\quad\quad + \alpha(t,\gamma(t))^2\cdot G''\left(\int_0^t\, \alpha(\tau,\gamma(\tau))\, d\tau \right).
	\end{split}
\end{equation*}
The second term in $A''$ is negative when $\mathcal{I}_R$ is submodular (as $G$ is concave down and $\alpha$ is positive). However, the first term can make $A''>0$. Notice this first term is strikingly similar to the information power. Let $\mathcal{K}_t$ be given as follows:
\begin{equation*}
	\mathcal{K}_t = G'\left(\int_0^t\, \alpha(\tau,\gamma(\tau))\, d\tau \right),
\end{equation*}
Then we have
\begin{equation*}
	A''(t) = \mathcal{K}_t\cdot d\alpha_t(v) + P(t),
\end{equation*}
where $P$ is the second (negative) term in $A''$. While the information power is given by $\rho(v) = \mathcal{K}_T\cdot d\alpha_t(v)$.
\subsection{Construction of a time-1 map}\label{sec:time1}
We want to construct a map $\varphi:T^*M\to T^*M$ which is a time-1 map for the memory flow. However, this is not straight-forward since the parameter $\mathcal{K}$ changes for each starting and final position. We get around this issue by the following way, \textit{although we do not assert that this is the canonical way to do so}. Recall that the value of $\mathcal{K}$ is given by
\begin{equation*}
	\mathcal{K} = \beta\left(\int_0^T\, \alpha(\tau,x(\tau))\, d\tau \right).
\end{equation*}
When $\mathcal{I}$ is submodular, this function is increasing in time. We can therefore construct a function $\kappa:T^*M\times\mathbb{R}\to \mathbb{R}$ such that
\begin{equation*}
	K = \beta\left( \int_0^{\kappa(x,p;K)} \, \alpha(\tau,x(\tau))\, d\tau \right).
\end{equation*}
We invert this function and get
\begin{equation*}
	K = \Lambda(x,p) \iff \kappa(x,p;K) = 1.
\end{equation*}
With this function $\Lambda$, we can finally construct a time-1 map:
$\varphi:T^*M\to T^*M$ where $\varphi(x,p)$ is the time-1 map of $X_\Lambda$ where
\begin{equation*}
	i_{X_\Lambda}\omega = dH + \Lambda\cdot \pi_M^*d\alpha_t.
\end{equation*}
\begin{remark}
    Optimization problems of the form \eqref{eq:origonal_problem} can also be handled via the maximum principle. An object of future work is to see how this definition of a time-1 map melds with the answers given by the maximum principle.
\end{remark}
\section{An Analytic Example}\label{sec:analytic}
We present an example where most items of interest have a closed-form solution. Consider the following problem.
\begin{equation}\label{eq:catenary}
	\arg\min \frac{m}{2}\, \int_0^T\, \dot{x}(t)^2 \, dt + \left(\int_0^T\, \frac{k}{2}x(t)^2 \, dt \right)^{1/2}.
\end{equation}
We will first consider solutions to this problem and then attempt to construct the time-1 map and determine whether or not it is symplectic.
\subsection{A solution}
Let us impose the boundary conditions $T=1$, $x(0)=0$, and $x(1)=1$. The MEL equation for this problem is 
\begin{equation}\label{eq:analytic_ODE}
	m\ddot{x} = \mathcal{K}\cdot kx, \quad \mathcal{K} = \frac{1}{2}\left( \int_0^1 \, \frac{k}{2}x(t)^2 \, dt \right)^{-1/2}.
\end{equation}
The solution to this ordinary differential equation is (treating $\mathcal{K}$ as a constant and using the prescribed boundary conditions above)
\begin{equation*}
	x_\mathcal{K}(t) = \frac{1}{\sinh{\sqrt{a}}}\cdot \sinh(\sqrt{a}t), \quad a = \frac{\mathcal{K}k}{m}.
\end{equation*}
Plugging in this solution (which depends on $\mathcal{K}$) back into the functional, we get
\begin{equation}\label{eq:J_K}
	\mathcal{J}(\mathcal{K}) = \frac{2\mathcal{K}k+A}{8\sinh^2\sqrt{a}} + \sqrt{ \frac{ A - 2\mathcal{K}k}{8\mathcal{K}\sinh^2\sqrt{a}}}, \quad 
	A = m\sqrt{a}\sinh(2\sqrt{a}).
\end{equation}
The optimal solutions corresponds to the maximizer of \eqref{eq:J_K}. This happens (by theorem \ref{thm:main}) when $\mathcal{K}$ satisfies \eqref{eq:analytic_ODE}. We need a fixed point of 
\begin{equation}\label{eq:f_K}
	\mathcal{K}^* = \frac{1}{2}\left(\int_0^T\, \frac{k}{2}x_\mathcal{K}(t)^2 \, dt \right)^{-1/2} = \frac{1}{2}\left( \frac{A - 2\mathcal{K}^*k}{8\mathcal{K}^*\sinh^2\sqrt{a}} \right)^{-1/2} =: f(\mathcal{K}).
\end{equation}

\begin{figure}
	\centering
	\begin{subfigure}[t]{0.45\textwidth}
		\centering
		\includegraphics[width=\textwidth]{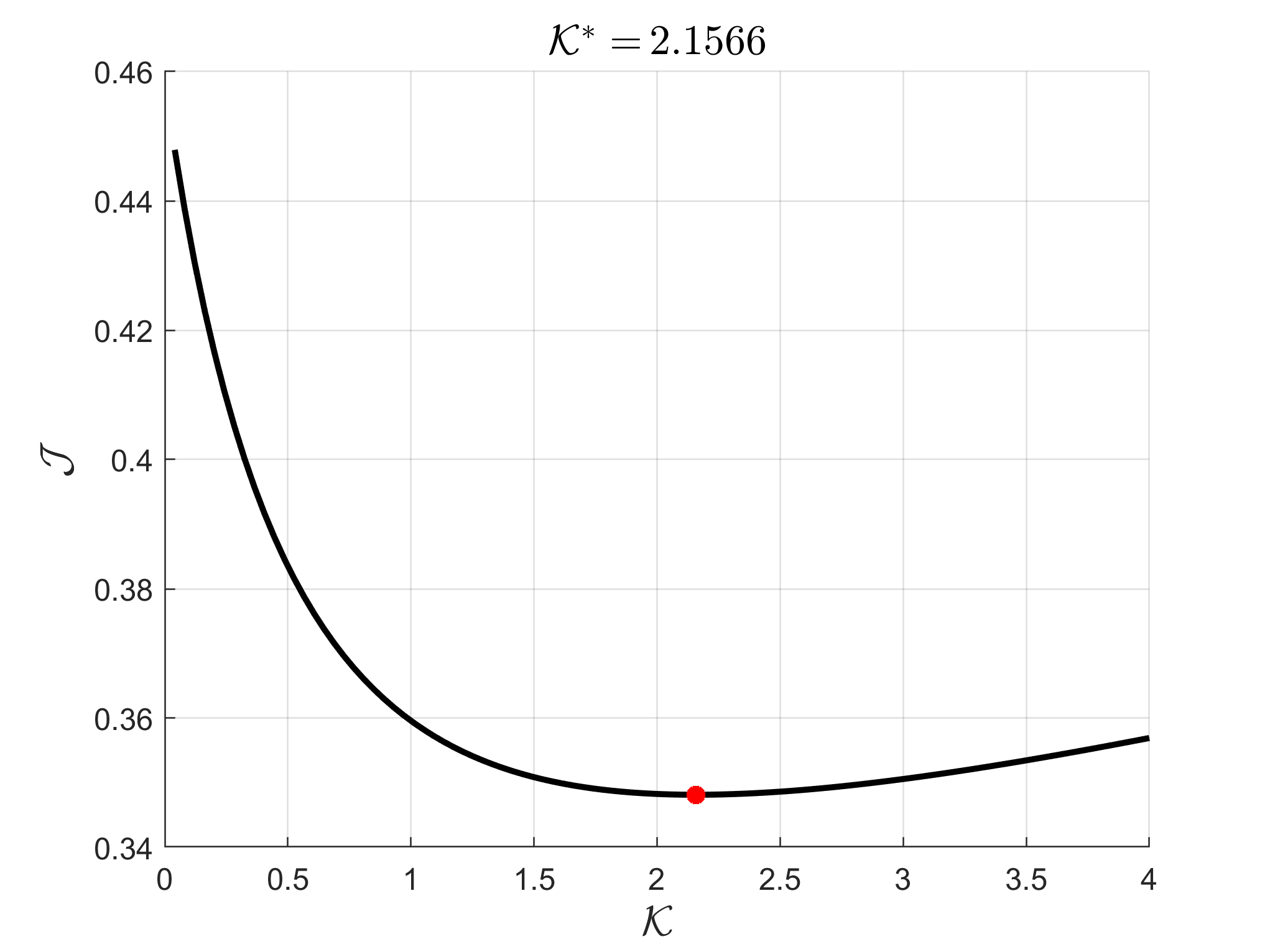}
		\caption{A plot of $\mathcal{J}$ from \eqref{eq:J_K}.}
		\label{fig:minimizer_catenary}
	\end{subfigure}
	\hfill
	\begin{subfigure}[t]{0.45\textwidth}
		\centering
		\includegraphics[width=\textwidth]{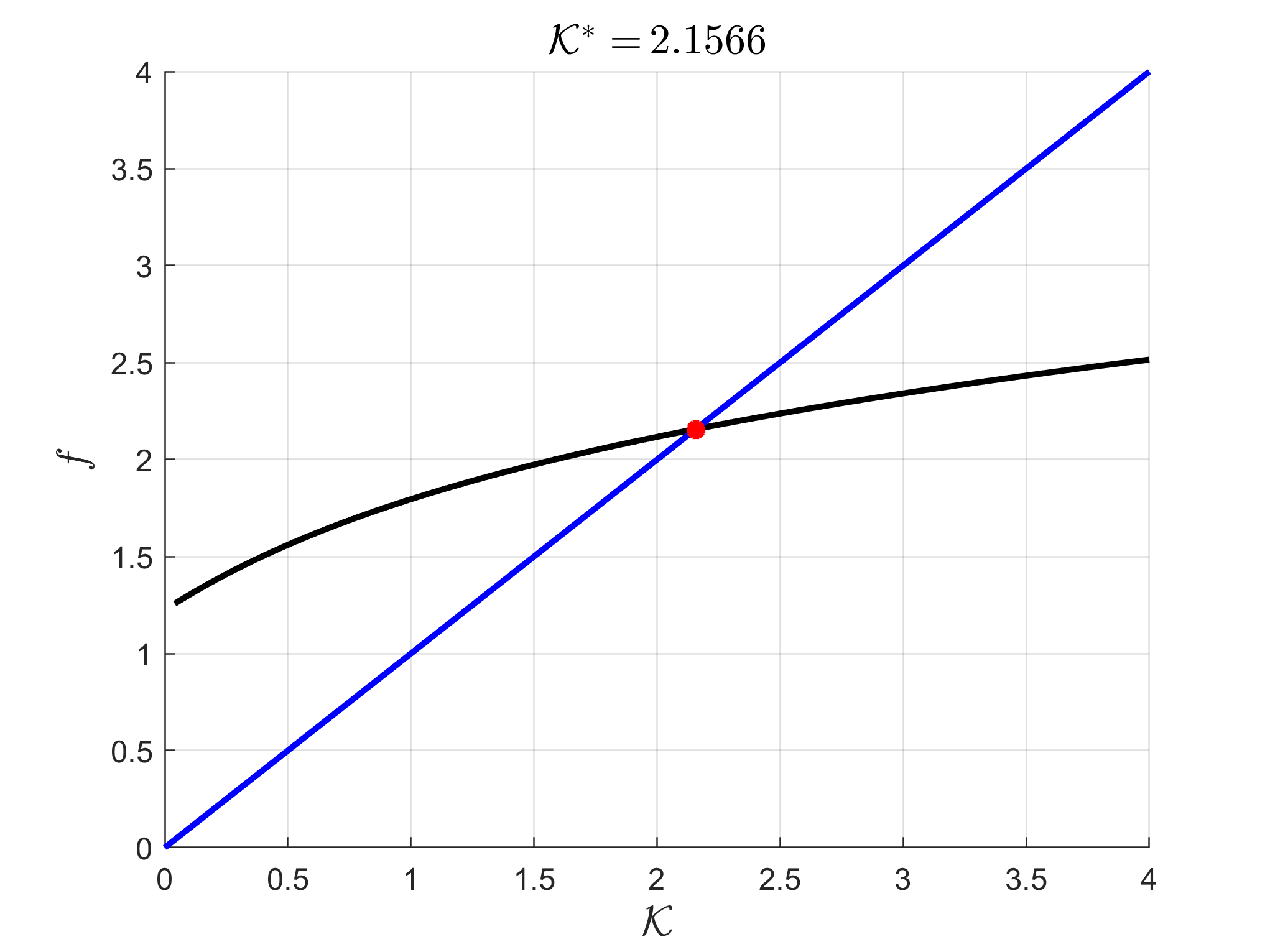}
		\caption{A plot of $f$ from \eqref{eq:f_K}.}
		\label{fig:fixed_catenary}
	\end{subfigure}
	\caption{The minimizer from \eqref{eq:J_K} agrees with the fixed-point from \eqref{eq:f_K} to obtain $\mathcal{K}^* = 2.1566$. This example used $m=0.1$ and $k=1$.}
	\label{fig:analytic_example}
\end{figure}

\subsection{The time-1 map}
We wish to determine the time-1 map for this example as outlined in section \ref{sec:time1}. The function $\Lambda$ has to satisfy
\begin{equation}\label{eq:Lambda_catenary}
	1 = \left[\Lambda(x,p)\right]^2\cdot C(x,p,\Lambda(x,p)),
\end{equation}
where $b= \sqrt{a}$ and
\begin{equation*}
	C(x,p,K) = \frac{2px}{K}\sinh^2(b) + kx^2 + \frac{kx^2}{2b}\sinh(2b) - \frac{p^2}{Km} + \frac{p^2}{2bKm}\sinh(2b).
\end{equation*}
We numerically implement two computations. First, we compute $\Lambda$ via \eqref{eq:Lambda_catenary}. Second, we compute $J$ where $J\cdot\omega = \varphi^*\omega$ with $\omega = dx\wedge dp$. In particular, the flow is symplectic if and only if $J\equiv 1$.

\begin{figure}
	\centering
	\begin{subfigure}[t]{0.45\textwidth}
		\centering
		\includegraphics[width=\textwidth]{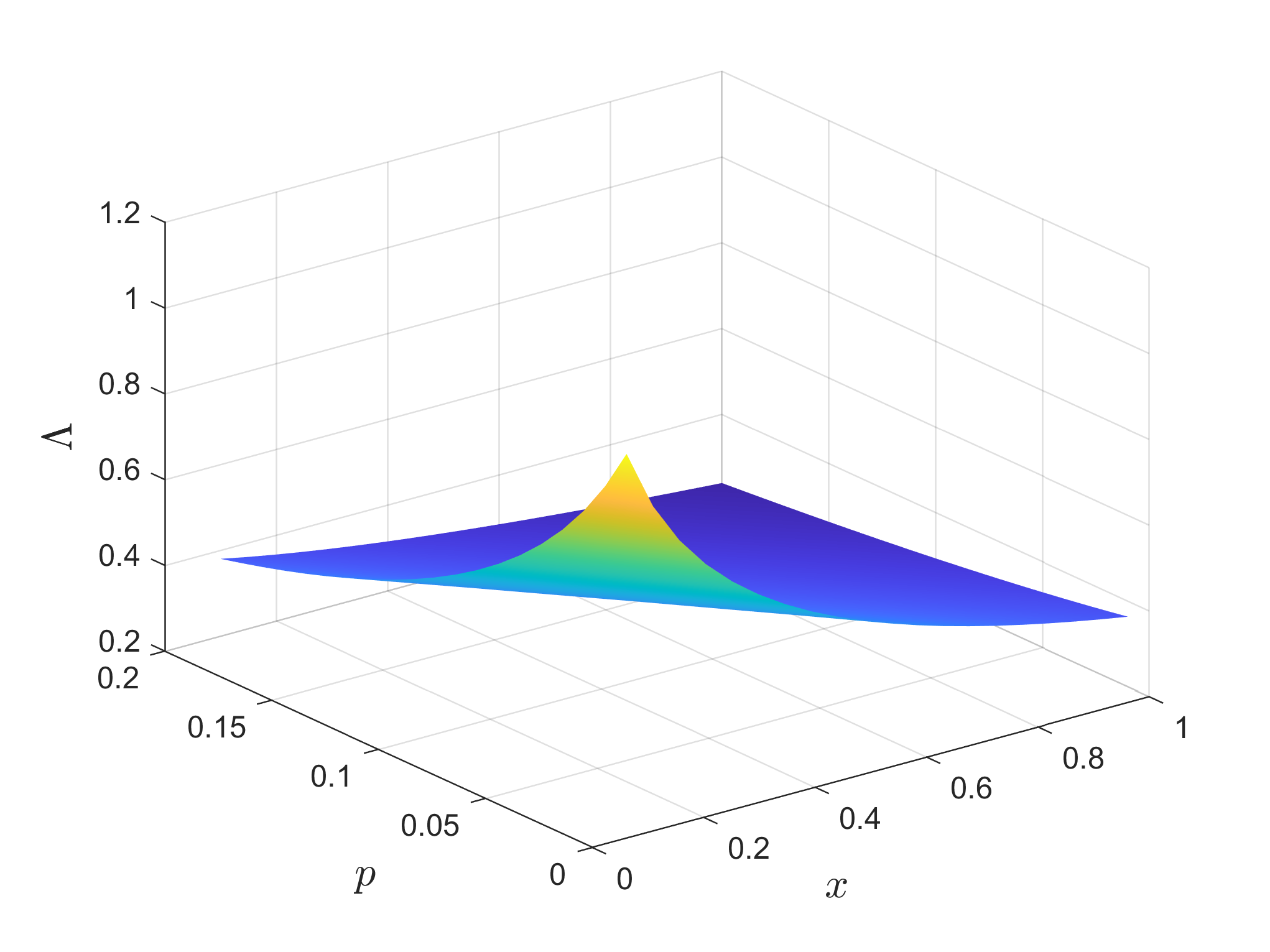}
		\caption{A plot of $\Lambda$ from \eqref{eq:Lambda_catenary}.}
		\label{fig:catenary_Lambda}
	\end{subfigure}
	\hfill
	\begin{subfigure}[t]{0.45\textwidth}
		\centering
		\includegraphics[width=\textwidth]{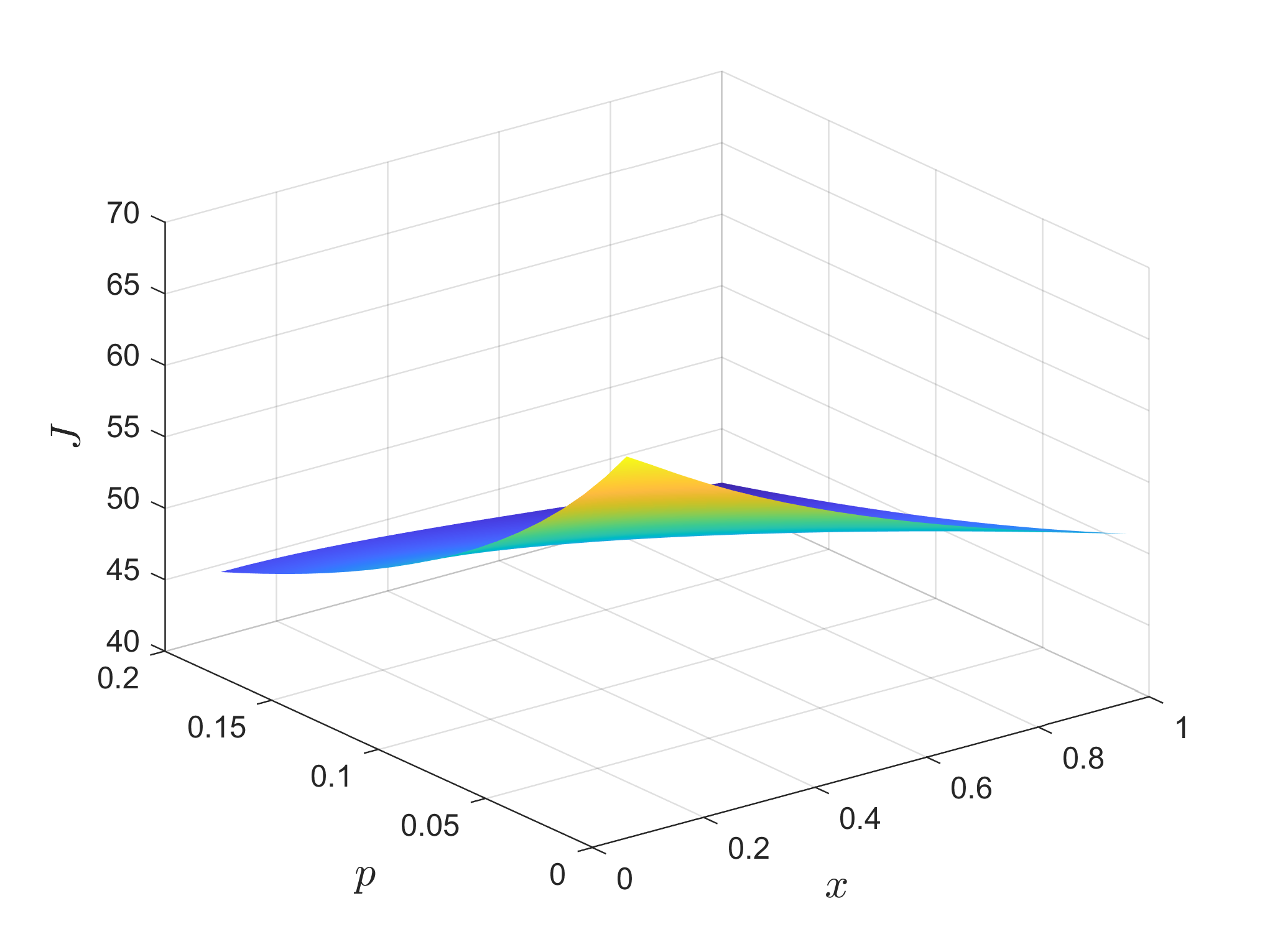}
		\caption{A plot of $J$ from $J\cdot\omega = \varphi^*\omega$. It is clear that $J\ne 1$ and therefore, the time-1 map is \textit{not symplectic}.}
		\label{fig:catenary_jacobian}
	\end{subfigure}
	\caption{Plots demonstrating that the time-1 map as discussed in section \ref{sec:time1} is not symplectic for the example given by \eqref{eq:catenary}. The parameters taken are $m=0.1$ and $k=1$. The derivative of $\varphi$ is computed via a difference quotient with a width of 0.001.}
	\label{fig:not_symplectic}
\end{figure}

\section{A Numerical Example}\label{sec:numerical}
The line example is useful for verifying the validity of \eqref{eq:example_MEL} by comparing it to the brute-force solution. We now proceed to a much more interesting example where $V=\mathbb{R}^2$. To do this, we will use \eqref{eq:full_map_dynamics}. However, determining the unknown function $\mathcal{K}_y$ will be quite difficult so we suppose that \textbf{there are only three points of interest}. The equations of motion are now
\begin{equation}
	m\ell^2\ddot{x} = \sum_{i}\, \nu(y_i)\cdot \mathcal{K}_{y_i}\cdot (x-y_i) \cdot k(x,y_i), \quad \mathcal{K}_{y_i}\in [0,1]
\end{equation}
For the purposes of this example, we will take the following:
\begin{equation*}
	x(0) = \begin{bmatrix}
		0 \\ 0
	\end{bmatrix}, \quad x(T) = \begin{bmatrix}
		0 \\ 2
	\end{bmatrix}, \quad y_1 = \begin{bmatrix}
		1 \\ 2
	\end{bmatrix}, \quad y_2 = \begin{bmatrix}
		2 \\ 0
	\end{bmatrix}, \quad y_3 = \begin{bmatrix}
		2 \\ 1
	\end{bmatrix},
\end{equation*}
and $\nu(y_i)=1$ for all $i$. The parameters for this example are: $m=0.1$, $T = 3$, $\ell=1$, and $\lambda = 1$.

Figure \ref{fig:numeric_example} shows the trajectory and the velocity plots. It is clear that the solution is attracted to informative points (i.e., points of interest) to maximize the collected information along the trajectory.  
\begin{figure}
	\centering
	\begin{subfigure}[t]{0.45\textwidth}
		\centering
		\includegraphics[width=\textwidth]{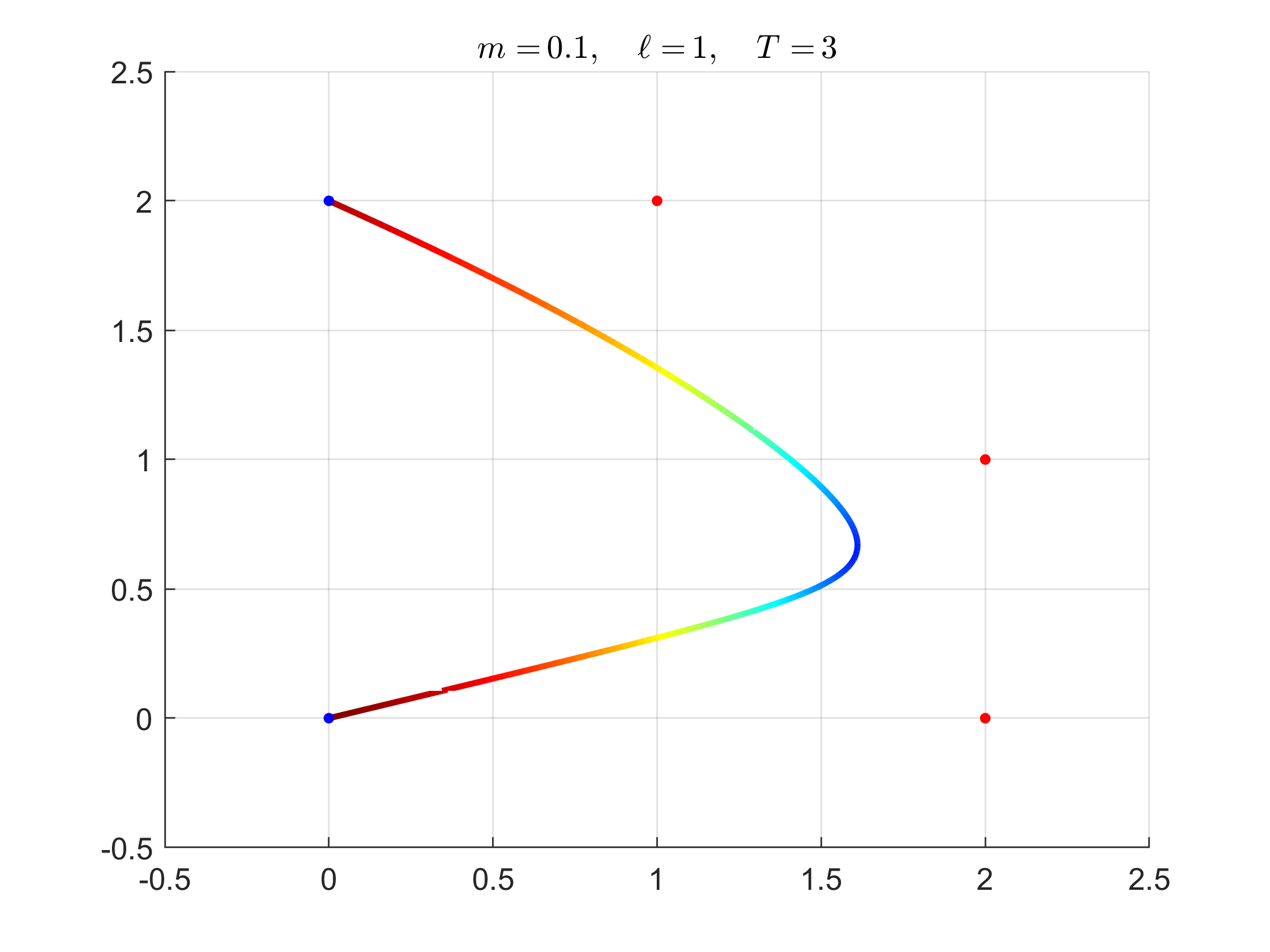}
		\caption{A plot of the trajectory. The color of the trajectory corresponds to its velocity, see figure \ref{fig:example_velocity}.}
		\label{fig:example_path}
	\end{subfigure}
	\hfill
	\begin{subfigure}[t]{0.45\textwidth}
		\centering
		\includegraphics[width=\textwidth]{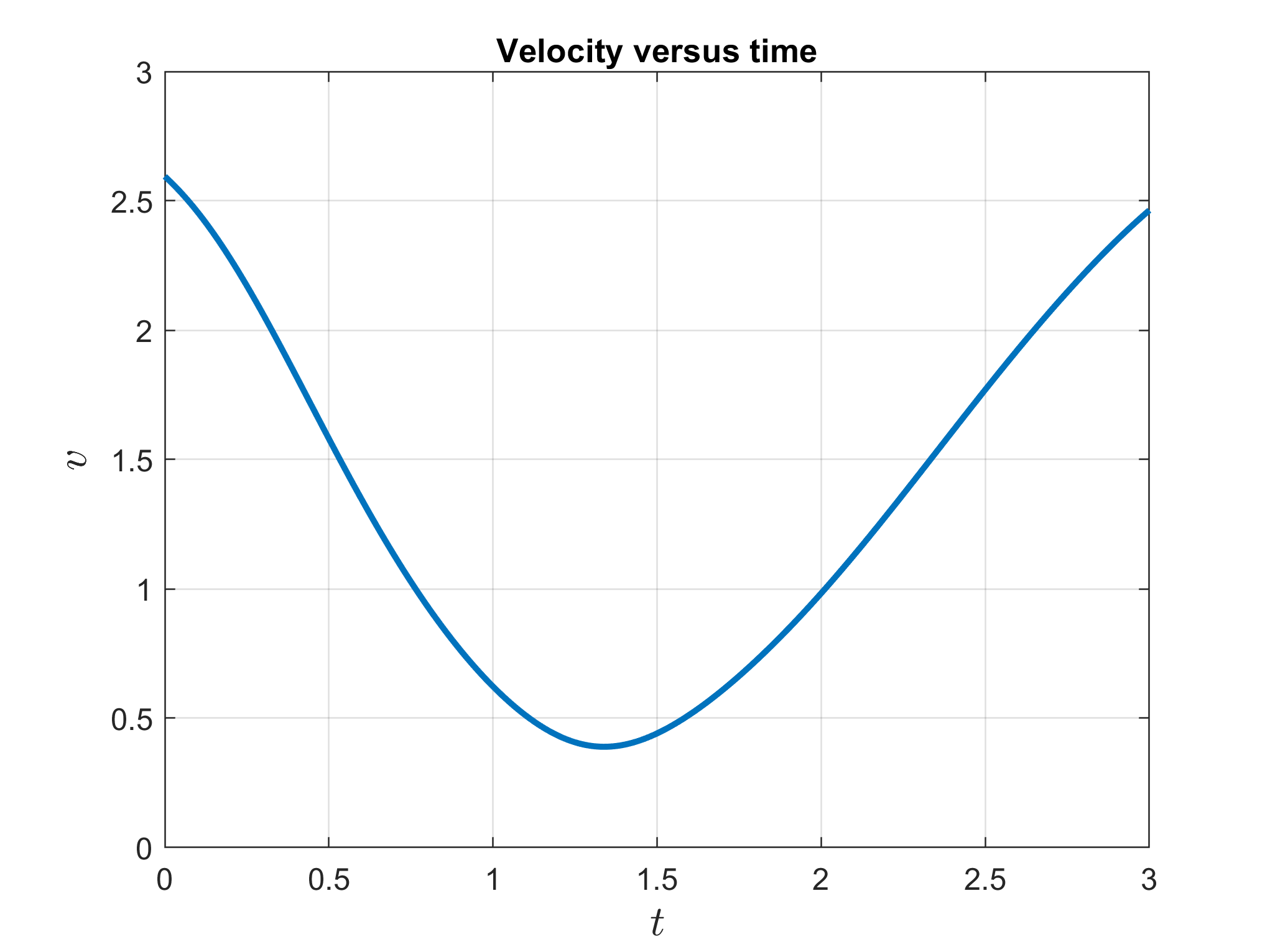}
		\caption{A plot of the velocity versus time.}
		\label{fig:example_velocity}
	\end{subfigure}
	\caption{Plots demonstrating the numerical example from section \ref{sec:numerical}.}
	\label{fig:numeric_example}
\end{figure}

\section{Conclusions}\label{sec:conclusions}
We developed a path-dependent variational framework to deal with submodular information functions in information gathering problems. The construction of the first-order necessary optimality conditions for a class of memory Lagrangians, under some non-restrictive regularity assumptions, resulted in the memory Euler-Lagrange equations. Moreover, we showed that these equations generalize the classical Euler-Lagrange equations when the Lagrangian does not depend on the past.

The numerical examples provided insight in how a trajectory for information maximization can be obtain. This result can form the building block of a more complex system for robotic exploration problems that we shall study in the future. Several interesting extensions of this paper are as follows.
\begin{itemize}
	\item The control problem where a system dynamics contraints the memory Lagrangian (recall that HJB does NOT work).
	\item The addition of a $\dot{\gamma}$ term to $\alpha$.
	\item An algorithmic implementation of the proposed framework and better strategies to handle $\mathcal{K}$.
	\item To compare with the maximum principle. Can our framework handle optimization problems that cannot be understood by invoking the maximum principle.
\end{itemize}

\section*{Acknowledgments}
We would like to thank Jessy W. Grizzle and Andy Borum for helpful conversations.

\bibliographystyle{plain}
\bibliography{references}
\end{document}